\documentclass[11pt]{article}
\usepackage{graphicx}
\usepackage{fullpage}
\usepackage{amsmath,amssymb,amsthm}
\usepackage{enumerate}
\usepackage[mathcal]{euscript}
\usepackage{comment}
\usepackage[round]{natbib}

\usepackage[T1]{fontenc}
\usepackage{microtype}
\usepackage{enumitem}

    \numberwithin{equation}{section}

    \usepackage[usenames]{color}
    \definecolor{plum}  {rgb}{.4,0,.4}
    \definecolor{BrickRed} {rgb}{0.6,0,0}
    
    \usepackage{commath}
    \usepackage[normalem]{ulem}

    \usepackage[plainpages=false,pdfpagelabels,colorlinks=true,linkcolor=BrickRed,citecolor=plum]{hyperref}

\def\ddefloop#1{\ifx\ddefloop#1\else\ddef{#1}\expandafter\ddefloop\fi}

\def\ddef#1{\expandafter\def\csname c#1\endcsname{\ensuremath{\mathcal{#1}}}}
\ddefloop ABCDEFGHIJKLMNOPQRSTUVWXYZ\ddefloop

\def\ddef#1{\expandafter\def\csname s#1\endcsname{\ensuremath{\mathsf{#1}}}}
\ddefloop ABCDEFGHIJKLMNOPQRSTUVWXYZ\ddefloop

\def\ddef#1{\expandafter\def\csname b#1\endcsname{\ensuremath{\mathbf{#1}}}}
\ddefloop ABCDEFGHIJKLMNOPQRSTUVWXYZabcdefghijklmnopqrstuvwxyz\ddefloop

\def\Reals{{\mathbb R}}

\def\Naturals{{\mathbb N}}

\def\deq{:=}
\def\eps{\varepsilon}
\def\trn{{\hbox{\textit{\tiny T}}}}

\def\fix#1{#1}

\def\Exp{{\mathbf E}}

\DeclareMathOperator*{\argmin}{arg\,min}

    \newtheorem{theorem}{Theorem}[section]
    \newtheorem{lemma}{Lemma}[section]

    \newtheorem{assumption}{Assumption}[section]

    \usepackage{cleveref}

\begin{document}

\title{Learning Recurrent Neural Net Models of Nonlinear Systems}

\author{Joshua Hanson\thanks{University of Illinois; e-mail: jmh4@illinois.edu.} \and Maxim Raginsky\thanks{University of Illinois; e-mail: maxim@illinois.edu.} \and Eduardo Sontag\thanks{Northeastern University; e-mail: e.sontag@northeastern.edu.}}

\date{}

\maketitle

\begin{abstract}%
    
We consider the following learning problem: Given sample pairs of input and output signals generated by an unknown nonlinear system (which is not assumed to be causal or time-invariant), we wish to find a continuous-time recurrent neural net with hyperbolic tangent activation function that approximately reproduces the underlying i/o behavior with high confidence. Leveraging earlier work concerned with matching output derivatives up to a given finite order  \citep{Sontag_recur}, we reformulate the learning problem in familiar system-theoretic language and derive quantitative guarantees on the sup-norm risk of the learned model in terms of the number of neurons, the sample size, the number of derivatives being matched, and the regularity properties of the inputs, the outputs, and the unknown i/o map.

\end{abstract}


\section{Introduction}

We consider a learning-theoretic framework for system identification, where the goal is to approximate an unknown nonlinear input/output (i/o) map by a continuous-time recurrent neural net (RNN) on the basis of a finite collection of input/output pairs. The approximation criterion is the $\cL^\infty$ norm of the difference between the ground-truth output and the one predicted by the learned model. 

Earlier work \citep{Sontag_recur} has addressed a related problem of reproducing output $k$-jets as a function of input $(k-1)$-jets via RNNs, where a $k$-\textit{jet} is defined as the vector of derivatives up to order $k$ of the output (respectively, input) signal evaluated at the initialization time. Equivalently, $k$-jets can be defined as optimal degree-$k$ polynomial approximations, i.e., truncated Taylor series. In this work, we build on the aforementioned result about matching $k$-jets, but work in a familiar system-theoretic setting using the language of i/o maps. 

The training procedure used both in that work and here differs to some extent from the standard ``backpropagation through time'' algorithm often used for training RNNs. Rather than forward-propagating training inputs through the net or setting up an adjoint equation for the net parameters, we ``pull back'' the input signals to the initialization time and compute the corresponding output $k$-jets, which can be expressed explicitly as a function of the weights of the neural net, its initial state, and the $(k-1)$-jets of the input. The loss function is then evaluated with the predicted and ground-truth output jets. At training time, there is no requirement to advance inputs through the network. In essence, the RNN can be trained using conventional nonlinear regression as if it were a single-layer feedforward net. 

In the following section, we describe the learning problem, state our assumptions, and introduce the needed notation. Proceeding in Section \ref{sec:learning_procedure}, we develop the setting of jets and use this language to formulate the proposed learning procedure and state our main results. Some directions for future research are also outlined. Section \ref{sec:proofs_prelim} contains some technical lemmas, followed by the proof of the main theorems in Sections \ref{sec:proof_riskbound} and \ref{sec:proof_main}.

\section{Problem formulation}

In this work, a \textit{system} (or an \textit{i/o map}) is a nonlinear operator $\sF : C([0,T]) \to C([0,T])$, where $C([0,T])$ is the Banach space of continuous functions $u : [0,T] \to \Reals$ equipped with the sup norm
\begin{align*}
    \|u\|_\infty \deq \sup_{t \in [0,T]}|u(t)|.
\end{align*}
The learning problem can be phrased as follows: Let $N$ pairs $(u^1,y^1),\ldots,(u^N,y^N)$ be given, where $u^i \in C([0,T])$ are the inputs and $y^i = \sF u^i \in C([0,T])$ are the corresponding outputs. We wish to construct a system $\hat{\sF} : C([0,T]) \to C([0,T])$ that approximately reproduces the unknown system $\sF$ on a given class of inputs $\cU \subset C([0,T])$. (We focus on single-input, single-output systems mainly for simplicity; our approach easily extends to multiple inputs and outputs.) Next, we impose a set of minimal assumptions on the system $\sF$ and specify the accuracy criterion. 

We first recall the definition of the \textit{modulus of continuity} of a function $u \in C([0,T])$ (see, e.g., Ch.~2  of \citet{DevLor93}):
\begin{align*}
    \omega_u(\delta) \deq \sup_{t_1,t_2 \in [0,T] \atop |t_1-t_2|\le \delta} |u(t_1)-u(t_2)|.
\end{align*}
The function $\delta \mapsto \omega_u(\delta)$ is nondecreasing, and, since any $u \in C([0,T])$ is uniformly continuous, $\lim_{\delta \downarrow 0}\omega_u(\delta) = \omega_u(0) = 0$. (In fact, we will refer to any function with these properties that majorizes $\omega_u$ as a modulus of continuity of $u$.) We impose the following assumption on the class of inputs $\cU$:

\begin{assumption}\label{as:equicontinuous}
The inputs in $\cU$ are uniformly bounded, i.e., $R \deq \sup_{ u \in \cU}\| u\|_\infty < \infty$,  and equicontinuous with common modulus of continuity $\omega_{\cU}(\delta)$:
\begin{equation*}
    \sup_{ u \in \cU} \omega_ u(\delta)\leq \omega_{\cU}(\delta),
\end{equation*}
where $\omega_\cU(\delta)$ is a nondecreasing function that satisfies $\lim_{\delta \downarrow 0}\omega_\cU(\delta) = \omega_\cU(0) = 0$.
\end{assumption}
We also use moduli of continuity to describe the regularity of $\sF$:
\begin{assumption}\label{as:outputs} The output space $\cY = \sF(\cU)$, i.e., the image of $\cU$ under $\sF$, is equicontinuous with a common modulus of continuity $\omega_\cY(\delta)$.
\end{assumption}
Finally, we assume that $\sF$ has the bounded-in, bounded-out property:
\begin{assumption}\label{as:bibo} $\gamma_\sF(R) \deq \displaystyle\sup_{u \in C[0,T] \atop \|u\|_\infty 
		\le R} \| \sF u \|_\infty < \infty$.\end{assumption}
Examples of classes of inputs that satisfy Assumption~\ref{as:equicontinuous} include:
\begin{itemize}
	\item finite combinations of Fourier terms
\begin{align}\label{eq:Fourier}
	u(t) = \sum^m_{i=1} c_i \sin(\omega_i t + \alpha_i),
\end{align}
for all $m \ge 1$, provided the coefficients $c_i$ and the frequencies $\omega_i$ satisfy the inequalities $\sum^m_{i=1}|c_i| \le R$ and $\sum^m_{i=1}|c_i\omega_i| \le L$ for some fixed finite constants $R$ and $L$;
\item polynomial inputs
\begin{align}\label{eq:poly}
	u(t) = \sum^m_{i=0} c_i t^i
\end{align}
for all $m \ge 0$, provided the coefficients $c_i$ satisfy the inequalities $\sum^m_{i=0} |c_i|T^i \le R$ and $\sum^m_{i=1} i|c_i| T^{i-1} \le L$ for some $R,L <\infty$.
\end{itemize}
In both cases, Assumption~\ref{as:equicontinuous} holds with $\omega_\cU(\delta) = L\delta$. As an example of a system satisfying our hypotheses, consider a state-space model of the form
\begin{align*}
	\dot{x}(t) &= f(x(t)) + g(x(t))u(t), \\
	y(t) &= h(x(t))
\end{align*}
where $x(t) \in \Reals^n$ is an internal state with a given initial condition $x(0) = \xi$. Let $\sF$ be the induced i/o map that sends the input $u : [0,T] \to \Reals$ to the output $y : [0,T] \to \Reals$. Then it is not hard to verify that Assumptions~\ref{as:outputs} and \ref{as:bibo} will hold if the functions $f : \Reals^n \to \Reals^n$, $g : \Reals^n \to \Reals^n$, and $h : \Reals^n \to \Reals$ are all Lipschitz continuous.

We assume the sample inputs $u^i$ are independent and identically distributed (i.i.d.)\ random elements of $C([0,T])$ drawn according to a fixed Borel probability measure $\mu$ supported on $\cU$. Thus, the input-output pairs $(u^1, y^1),\ldots,(u^N, y^N)$, with $y^i = \sF u^i$, are themselves i.i.d.\ random elements of the product space $\cU \times \cY$. For instance, $\sF$ could be a model of a two-terminal electronic device whose terminals are connected through a switch to an excitation circuit consisting of linear and nonlinear elements and current and/or voltage sources \citep{Chua80}. Suppose this excitation circuit is specified by a vector of parameters $\theta \in \Reals^d$; e.g., it could be used to generate sinusoidal or polynomial inputs, as in \eqref{eq:Fourier} or \eqref{eq:poly}. We can then generate $N$ i.i.d.\ samples $\theta^1,\ldots,\theta^N$ according to a fixed Borel probability measure $\nu$ on $\Reals^d$. Each sample $\theta^i$ corresponds to a random realization of the excitation circuit so that, when the switch is closed at time $t=0$ and open at time $t=T$, we can take the input $ u^i : [0,T] \to \Reals$ to be the voltage waveform across $\sF$ and the output $ y^i : [0,T] \to \Reals$ to be the current waveform through $\sF$. (This assumes that $\sF$ is voltage-controlled.) Thus, the probability measure on the input-output space $\cU \times \cY$ is well-defined but specified \textit{indirectly} through $\nu$, the structure of the excitation circuit, and $\sF$. At any rate, given $\sF$ and an approximating system $\hat{\sF}$, we define the $\cL^\infty$ risk
\begin{align*}
    \cL(\hat{\sF}) \deq \Exp_\mu \big[\| \hat{\sF} u - \sF  u \|_\infty\big],
\end{align*}
where the expectation is taken with respect to the probability measure $\mu$ on the input space $\cU$. The goal is to generate, on the basis of the observed input-output pairs $(u^i, y^i)$, an approximate system $\hat{\sF}$ from a given model class $\cF$, so that the risk $\cL(\hat{\sF})$ (which is a random variable due to its dependence on the training data) is small with high probability.

\section{The proposed learning procedure and its performance}\label{sec:learning_procedure}

\paragraph{Recurrent neural nets.} We first specify the model class $\cF$ from which our learning procedure will select the approximation $\hat{\sF}$. The systems in $\cF$ are described by differential equations of the form
\begin{subequations}\label{eq:RNN}
\begin{align}
    \dot{x}(t) &= \sigma^{(n)}(Ax(t)+b u(t)) \\
    y(t) &= c^\trn x(t)
\end{align}
\end{subequations}
for $t \in [0,T]$ with initial condition $x(0) = \xi \in \Reals^n$. Here, $A \in \Reals^{n \times n}$, $b,c \in \Reals^n$, and $\sigma^{(n)} : \Reals^n \to \Reals^n$ is the diagonal map $\sigma^{(n)}((x_1,\ldots,x_n)^\trn) \deq (\sigma(x_1),\ldots,\sigma(x_n))^\trn$ with $\sigma(r) = \tanh(r)$. The system \eqref{eq:RNN} is a continuous-time \textit{recurrent neural net} (RNN) with $n$ neurons. Each pair $(\Sigma,\xi)$, where $\Sigma \deq (A,b,c)$, specifies an input-output map $\sF_{\Sigma,\xi}$ that sends an input $ u \in C([0,T])$ to the corresponding output $ y$ given by
\begin{align*}
    y(t) = c^\trn \xi + \int^t_0 c^\trn \sigma^{(n)}(A x(\tau) + b u(\tau)) \dif \tau, \qquad t \in [0,T].
\end{align*}
For $0 < M < \infty$, we define the class of systems
\begin{align*}
    \cF(M) \deq \left\{ \sF_{\Sigma,\xi} = \sF_{(A,b,c),\xi} : \|A\|,|b|,|c|,|\xi| \le M \right\},
\end{align*}
where $\|A\|$ is the spectral norm of $A$, and $|b|,|c|,|\xi|$ are the $\ell^2$ norms of $b,c,\xi$. Our learning procedure will generate a pair $(\hat{\Sigma},\hat{\xi})$ based on the data $\{(u^i, y^i)\}$, and output the model $\hat{\sF} = \sF_{\hat{\Sigma},\hat{\xi}} \in \cF(M)$. With a slight abuse of notation, we will often write $(\Sigma,\xi) \in \cF(M)$ instead of $\sF_{\Sigma,\xi} \in \cF(M)$.

\paragraph{Jets.} While it is well-known that RNNs of the form \eqref{eq:RNN} are universal approximators for i/o maps $\sF$ that admit smooth nonlinear state-space realizations $\dot{x} = f(x,u)$, $y = g(x)$ \citep{sontag_1992,funahashi_1993,HansonRaginskyRNN}, here we are \textit{not} assuming that $\sF$ admits such a realization (in fact, we are not even requiring $\sF$ to be causal or time-invariant). Nevertheless, we will show that, provided Assumptions~\ref{as:equicontinuous}--\ref{as:bibo} hold for our learning problem, we will be able to approximate $\sF$ by a recurrent net model  which will have the properties of causality and time invariance by construction. Our approach proceeds by way of reducing the infinite-dimensional problem of learning the i/o map $\sF_{\hat{\Sigma},\hat{\xi}}$ to a certain finite-dimensional problem \citep{Sontag_recur}. To that end, consider a system of the form \eqref{eq:RNN} fed with an input $ u$ which has at least $k-1$ derivatives at $t=0$. Then the output $ y = \sF_{\Sigma,\xi} u$ will have at least $k$ derivatives at $t=0$, which can be computed explicitly as
\begin{align*}
    y^{(0)}(0) &= c^\trn \xi, \quad y^{(1)}(0) = c^\trn \sigma^{(n)}(A\xi + b u(0)), \quad  \ldots.
\end{align*}
We can then define the map $Y_{k,\Sigma,\xi} : \Reals^k \to \Reals^{k+1}$ according to
\begin{align*}
    Y_{k,\Sigma,\xi}((u(0),u'(0),\ldots,u^{(k-1)}(0))^\trn) \deq (y(0),y'(0),\ldots,y^{(k)}(0))^\trn,
\end{align*}
where $y^{(\ell)}(0) =\frac{\dif^{\,\ell}}{\dif t^\ell}\big|_{t=0} \sF_{\Sigma,\xi} u(t)$ for $0 \le \ell \le k$. We can also phrase this in terms of \textit{jets}, where the $k$-jet at $t=0$ of a function $f : \Reals \to \Reals$ which is $C^k$ in some neighborhood of $t=0$ is the degree-$k$ polynomial
\begin{align*}
    J^k_0 f(s) \deq \sum^k_{\ell=0} f^{(\ell)}(0) \frac{s^\ell}{\ell!}.
\end{align*}
Then, for inputs that are $C^{k-1}$ in some neighborhood of $t=0$, the map $Y_{k,\Sigma,\xi} : \Reals^k \to \Reals^{k+1}$ can be lifted to a map from ${(k-1)}$-jets to $k$-jets via $J^{k-1}_0 u \mapsto J^{k}_0 \sF_{\Sigma,\xi}u$, where the vector of coefficients of $J^k_0 \sF_{\Sigma,\xi} u$ is the image of the vector of coefficients of $J^{k-1}_0 u$ under $Y_{k,\Sigma,\xi}$. Thus, at least for inputs $u$ that are of class $C^{k-1}$, the i/o map $\sF_{\Sigma,\xi}$ will be a good approximation to $\sF$ provided (a) the outputs $\sF u$ and $\sF_{\Sigma,\xi} u$ can be accurately approximated by their $k$-jets $J^k_0 \sF u$ and $J^k_0 \sF_{\Sigma,\xi}u$ (e.g., if $k$ is sufficiently large) and (b) the $k$-jets $J^k_0 \sF u$ and $J^k_0 \sF_{\Sigma,\xi} u$ are close to one another in sup norm on $[0,T]$. 

\paragraph{The learning procedure.} Since the inputs in $\cU$ and the corresponding outputs in $\cY$ are only assumed to be continuous, we will first approximate them by functions that have as many derivatives at $t=0$ as needed. To that end, we start by defining for each $k \in \Naturals$ two linear maps $\sS_k : C([0,T]) \to \Reals^k$ and $\sS^*_k : \Reals^k \to C([0,T])$
according to
\begin{align*}
    \sS_k(u) := \Big(B_{k-1}( u,0), \frac{\dif}{\dif t}{\Big|}_{t=0} B_{k-1}(u,t), \dots, \frac{\dif^{\,k-1}}{\dif t^{k-1}}{\Big|}_{t=0} B_{k-1}( u,t) \Big)^\trn,
\end{align*}
and
\begin{align*}
          \sS^*_k((a_0,\dots,a_{k-1})^\trn)(t) := \sum_{\ell=0}^{k-1} a_\ell \frac{t^\ell}{\ell !},
\end{align*}
where
\begin{align*}
    B_{m}(u,t) \deq \sum^m_{i=0} u\left(\frac{i T}{m}\right) {m \choose i} \left(\frac{t}{T}\right)^i \left(1-\frac{t}{T}\right)^{m-i}, \qquad t \in [0,T]
\end{align*}
is the degree-$m$ \textit{Bernstein polynomial} of $u \in C([0,T])$ \citep[Ch.~1]{DevLor93}. Whenever no confusion will arise, we will also write $B_m u$ instead of $B_m(u,\cdot)$.
Observe that  $(\sS^*_k \circ \sS_k)u = B_{k-1}(u,\cdot)$. To motivate the introduction of these objects, we give the following bound on the expected risk of any recurrent net model $\sF_{\Sigma,\xi}$:

\begin{theorem}\label{thm:riskbound} Under Assumptions~\ref{as:equicontinuous} and \ref{as:outputs}, the following holds for any $\Sigma = (A,b,c)$, $\xi$, and $k \ge 2$:
    \begin{align}\label{eq:riskbound}
    \begin{split}
        \cL(\sF_{\Sigma,\xi}) &\le 2\omega_\cY\left(\frac{T}{\sqrt{k}}\right) +2|c| |b|e^{\|A\|T} \omega_\cU\left(\frac{2T}{\sqrt{k}}\right) + \fix{4}|c| Te^{\|A\|T} \sqrt{\frac{n}{k}} \\
		& \qquad \qquad \fix{ + \Exp_\mu [ \| J^k_0 \sF_{\Sigma,\xi} B_{k-1}u - \sF_{\Sigma,\xi} B_{k-1}u\|]} \\
        & \qquad \qquad + \Exp_\mu \left[ \| (\sS^*_{k+1} \circ Y_{k,\Sigma,\xi} \circ \sS_k)u - (\sS^*_{k+1} \circ \sS_{k+1} \circ \sF )u  \|_\infty \right].
    \end{split}
    \end{align}
\end{theorem}
\noindent Note that the last term on the right-hand side of \eqref{eq:riskbound} is the expectation, with respect to $u \sim \mu$, of the sup norm of the difference between the degree-$k$ Bernstein polynomial $B_k \sF u$ and the \fix{$k$-jet $J^k_0\sF_{\Sigma,\xi}B_{k-1}u$.}

We are now ready to present our learning procedure. For $1 \le j \le k$, let $t_j \deq jT/k$, and consider the following Empirical Risk Minimization (ERM) scheme:
\begin{align}\label{eq:RNN_ERM}
    (\hat{\Sigma},\hat{\xi}) \in \argmin_{(\Sigma,\xi) \in \cF(M)} \frac{1}{N}\sum^N_{i=1} \max_{1 \le j \le k} \left| (\sS^*_{k+1} \circ Y_{k,\Sigma,\xi} \circ \sS_k)u(t_j) - (\sS^*_{k+1} \circ \sS_{k+1})y(t_j)\right|.
\end{align}
The objective being minimized in \eqref{eq:RNN_ERM} is simply the empirical expectation of the maximum absolute difference between the \fix{degree-$k$ polynomials $B_k \sF u$ and $J^k_0 \sF_{\Sigma,\xi} B_{k-1} u$} on the finite grid $\{T/k, 2T/k, \ldots, (k-1)T/k, T\} \subset [0,T]$.

\begin{theorem}\label{thm:main} Suppose Assumptions~\ref{as:equicontinuous}--\ref{as:bibo} are satisfied, and $N \ge k(6 n^6 + 10 n^3 \log_2 k)$. Then with probability at least $1-\delta$, 
    \begin{align}\label{eq:ERMrisk}
    \begin{split}
        \cL(\hat{\sF}) &\le 4\omega_\cY\left(\frac{T}{\sqrt{k}}\right) +2M^2 e^{M T} \omega_\cU\left(\frac{2T}{\sqrt{k}}\right) + \fix{6} M Te^{MT} \sqrt{\frac{n}{k}} \\
		& \qquad \qquad \fix{+ 2\sup_{(\Sigma,\xi) \in \cF(M)} \Exp_\mu [\| J^k_0 \sF_{\Sigma,\xi}B_{k-1}u - \sF_{\Sigma,\xi}B_{k-1}u \|_\infty] }\\
        & \qquad \qquad + \bar{\cL}^* + c \left( M(M+\sqrt{n}T) + \gamma_\sF(R) \right) \sqrt{\frac{k(n^6 + n^3 \log_2 k)\log N + \log(\frac{1}{\delta})}{N}},
    \end{split}
    \end{align}
    where $c > 0$ is an absolute constant and
    \begin{align}\label{eq:minrisk}
        \bar{\cL}^* \deq \inf_{(\Sigma,\xi) \in \cF(M)} \Exp_\mu \left[ \max_{1 \le j \le k} \left| (\sS^*_{k+1} \circ Y_{k,\Sigma,\xi} \circ \sS_k)u(t_j) - (\sS^*_{k+1} \circ \sS_{k+1})y(t_j)\right| \right].
    \end{align}
\end{theorem}
The first three terms in \eqref{eq:ERMrisk} can be made arbitrarily small by choosing $k$ sufficiently large. The exponential dependence on $M$ and $T$ is an unavoidable artifact of the Gr\"{o}nwall lemma, and can be removed under appropriate stability assumptions \citep{HansonRaginskyRNN}. \fix{The next term is the worst-case, over $(\Sigma,\xi) \in \cF(M)$, expected error between $\sF_{\Sigma,\xi}B_{k-1}u$ (i.e., the response of the RNN $\sF_{\Sigma,\xi}$ to the polynomial input $B_{k-1}u$) and its $k$-jet. This term will only vanish with increasing $k$ if all the functions $\sF_{\Sigma,\xi}B_{k-1}u$ are entire.} The remaining two terms are, respectively, the approximation error and the estimation error of the ERM procedure \eqref{eq:RNN_ERM}. While the latter can be made arbitrarily small by increasing the size $N$ of the training set, the former is an intrinsic measure of the ability of recurrent neural nets to approximate output $k$-jets for a randomly chosen input. This minimal error value can be decreased by considering larger, more expressive nets; a quantitative analysis of this term is a promising direction for further work. Note also that, for a fixed $k$, we only need to collect input and output samples on an equispaced grid $\{0,T/k,2T/k, \ldots, T\}$. On the other hand, in many applications (e.g., medical or electronic system modeling), the input/output data are often available at non-uniformly spaced times $0 \le t_1 < t_2 < \ldots < t_k \le T$ that are not under the learner's control. Extending the approach of this paper to non-uniform (or even random) sampling of time instants is another interesting future direction.

\section{Proofs}

\subsection{Technical lemmas}
\label{sec:proofs_prelim}

In this section, we collect a few results that will be used in the proofs of Theorems~\ref{thm:riskbound} and \ref{thm:main}.

\begin{lemma}\label{lm:Bernstein} For any $u \in C([0,T])$ with modulus of continuity $\omega$ and any $k \in \Naturals$, the Bernstein polynomial $B_k(u,\cdot)$ has modulus of continuity $2\omega$ and satisfies
    \begin{align}\label{eq:Bernstein_apx}
        \| u - B_k(u,\cdot) \|_\infty \le 2\omega \left(\frac{T}{\sqrt{k}}\right).
    \end{align}
\end{lemma}
\begin{proof} The statement about the modulus of continuity of $B_k(u,\cdot)$ is a result of \cite{Li_Bernstein}. The error bound \eqref{eq:Bernstein_apx} is well-known, but we give a self-contained probabilistic proof. Let $\bar{\omega}$ be the smallest concave majorant of $\omega$. Then $\bar{\omega}$ is also a modulus of continuity of $u$, which further satisfies the inequality $\omega(\delta) \le \bar{\omega}(\delta) \le 2\omega(\delta)$ \cite[Lemma~6.1]{DevLor93}.

For any $t \in [0,T]$, we can express the value $B_k(u,t)$ as an expectation
    \begin{align*}
        B_k(u,t) = \Exp\Big[u\Big(\frac{T}{k}X_t\Big)\Big]
    \end{align*}
    where $X_t \sim \mathrm{Bin}(k,t/T)$. Then we have the following chain of (in)equalities:
    \begin{align*}
        |u(t)-B_k(u,t)| &= \left|u(t) - \Exp\Big[u\Big(\frac{T}{k}X_t\Big)\Big] \right| \\
        &\stackrel{\mathrm{(a)}}{\le} \Exp\left[\left|u(t)-u\Big(\frac{T}{k}X_t\Big)\right|\right] \\
        &\stackrel{\mathrm{(b)}}{\le} \Exp \left[\bar{\omega}\left(\left|\frac{T}{k}X_t - t\right|\right)\right] \\
        &\stackrel{\mathrm{(c)}}{\le} \bar{\omega} \left( \Exp\left|\frac{T}{k}X_t -t\right|\right) \\
        &\stackrel{\mathrm{(d)}}{\le} 2\omega \left(\frac{T}{\sqrt{k}}\right),
    \end{align*}
	where (a) is by Jensen's inequality, (b) follows from the fact that $\bar{\omega}$ is a modulus of continuity of $u$, (c) is by Jensen's inequality and by concavity of $\bar{\omega}$, and (d) uses the monotonicity of $\omega$ and $\bar{\omega}$, the inequality $\bar{\omega} \le 2\omega$, and the inequality $\Exp |U - \Exp U| \le \sqrt{k p(1-p)}$ for the mean absolute deviation of a $\mathrm{Bin}(k,p)$ random variable $U$.
\end{proof}

\begin{lemma}\label{lm:SS_apx} Let $\sG : C([0,T]) \to C([0,T])$ be a Lipschitz-continuous i/o map, i.e.,
    \begin{align*}
        \|\sG\|_\mathrm{Lip} \deq \sup_{u_1,u_2 \in C([0,T]) \atop u_1 \neq u_2} \frac{\| \sG u_1 - \sG u_2 \|_\infty}{\|u_1 - u_2 \|_\infty} < \infty.
    \end{align*}
    Then, for any integer $k \ge 2$ and any $u \in C([0,T])$,
    \begin{align*}
        \| \sG u - (\sS^*_{k+1} \circ \sS_{k+1} \circ \sG \circ \sS^*_k \circ \sS_k) u \|_\infty \le 2\|\sG\|_\mathrm{Lip} \omega_u\left(\frac{2T}{\sqrt{k}}\right) + 2\omega_{\sG B_{k-1}u}\left(\frac{T}{\sqrt{k}}\right).
    \end{align*}
\end{lemma}
\begin{proof} By the triangle inequality, 
\begin{align*}
\begin{split}
    &\phantom{{}={}} \| \sG u - (\sS^*_{k+1} \circ \sS_{k+1} \circ \sG \circ \sS^*_k \circ \sS_k) u \|_\infty \\
    &\qquad \le \| \sG u - (\sG \circ \sS^*_k \circ \sS_k) u \|_\infty + \| (\sG \circ \sS^*_k \circ \sS_k)u - (\sS^*_{k+1} \circ \sS_{k+1} \circ \sG \circ \sS^*_k \circ \sS_k)u \|_\infty \\
    &\qquad =: T_1 + T_2.
\end{split}
\end{align*}
Since $(\sS^*_k \circ \sS_k) u = B_{k-1}u$, Lemma~\ref{lm:Bernstein} gives
\begin{align*}
    T_1 &\le \|\sG\|_\mathrm{Lip} \| u - B_{k-1}u \|_\infty \le 2\|\sG\|_\mathrm{Lip} \omega_u \left(\frac{T}{\sqrt{k-1}}\right) \le 2\|\sG\|_\mathrm{Lip} \omega_u \left(\frac{2T}{\sqrt{k}}\right),
\end{align*}
where the last inequality follows from the fact that $\sqrt{k-1} \ge \frac{1}{2}\sqrt{k}$ for $k \ge 2$. Similarly,
\begin{align*}
    T_2 &= \| \sG B_{k-1}u - B_k\sG B_{k-1}u \|_\infty \le 2\omega_{\sG B_{k-1}u} \left(\frac{T}{\sqrt{k}}\right).
\end{align*}
\end{proof}

\begin{lemma}\label{lm:Hmap} Let $\sG$ be the i/o map of the differential dynamical system
    \begin{align}\label{eq:diffsys}
        \dot{x}(t) = f(x(t),u(t)), \quad y(t) &= h^\trn x(t); \qquad t \in [0,T],\, x(0) = \xi
    \end{align}
with input $u(t) \in \Reals$, state $x(t) \in \Reals^n$, and output $y(t) \in \Reals$. Suppose that $f(\cdot,\cdot)$ is bounded, i.e., $\sup_{(x,u) \in \Reals^n \times \Reals} |f(x,u)| < \infty$, and Lipschitz-continuous, i.e.,
\begin{align*}
    |f(x_1,u_1)-f(x_2,u_2)| \le L_X |x_1 - x_2| + L_U |u_1-u_2|.
\end{align*}
Then, for any input $u \in C([0,T])$, the output $y = \sG u$ is bounded with
\begin{align}\label{eq:ysup}
	\| y \|_\infty \le |h|\left(|\xi| + T\sup_{x,u}|f(x,u)|\right),
\end{align}
and has modulus of continuity
\begin{align}\label{eq:ymod}
    \omega_y(\delta) \le \frac{|h|}{L_X}\sup_{x,u}|f(x,u)|\left(e^{L_X \delta}-1\right).
\end{align}
Moreover, the i/o map $\sG$ is Lipschitz-continuous, with
\begin{align}\label{eq:HLip}
    \|\sG\|_\mathrm{Lip} \le \frac{|h|L_U }{L_X}\left(e^{L_XT}-1\right).
\end{align}
\end{lemma}
\begin{proof} For a given input $u \in C([0,T])$ and vector $\xi \in \Reals^n$, let $\varphi^u_{s,t}(\xi)$, $0 \le s \le t \le T$, denote the \textit{flow} of \eqref{eq:diffsys}, i.e., the solution of the ODE
    \begin{align*}
        \frac{\dif}{\dif t} \varphi^u_{s,t}(\xi) = f(\varphi^u_{s,t}(\xi),u(t)), \qquad s \le t \le T,\, \varphi^u_{s,s}(\xi) = \xi.
    \end{align*}
The flow map evidently has the semigroup property
\begin{align}\label{eq:semiflow}
    \varphi^u_{s,t}(\xi) = \varphi^u_{r,t}(\varphi^u_{s,r}(\xi)), \qquad 0 \le s \le r \le t \le T.
\end{align}
In terms of the flow, the i/o map $\sG$ is given by $y(t) = (\sG u)(t) = h^\trn \varphi^u_{0,t}(\xi)$. To show that $y$ is bounded, we first estimate
\begin{align*}
	|\varphi^u_{0,t}(\xi)| = \left|\xi + \int^t_0 f(\varphi^u_{0,\tau}(\xi),u(\tau))\dif \tau\right| \le |\xi| + T \sup_{x,u}|f(x,u)|, \qquad 0 \le t \le T
\end{align*}
and then $|y(t)| \le |h||\varphi^u_{0,t}(\xi)|$ by Cauchy--Schwarz. This gives \eqref{eq:ysup}.

To obtain the modulus of continuity of $y$, consider two times $0 \le t_1 \le t_2 \le T$. Then, using \eqref{eq:semiflow} and the time invariance of \eqref{eq:diffsys}, we can write
\begin{align*}
    | \varphi_{0,t_2}^ u(\xi) - \varphi_{0,t_1}^ u(\xi) | &= | \varphi_{t_1,t_2}^ u(\varphi_{0,t_1}^ u(\xi)) - \varphi_{0,t_1}^ u(\xi) | \\
    &= | \varphi_{0,t_2 - t_1}^{ \bar{u}}(\xi') - \xi' |,
\end{align*}
where $\xi' := \varphi_{0,t_1}^ u(\xi)$ and $\bar{u}(t)$ is any continuous extension of the map $[0,t_2-t_1] \ni t \mapsto  u(t+ t_1)$ to $[0,T]$. Applying the fundamental theorem of calculus and the triangle inequality, we have
\begin{align*}
    &\phantom{{}={}} | \varphi_{0,t_2 - t_1}^{ \bar{u}}(\xi') - \xi' | \\
    &\qquad= \Big| \int_0^{t_2 - t_1} f(\varphi_{0,\tau}^{ \bar{u}}(\xi'),  \bar{u}(\tau)) \dif\tau \Big| \\
    &\qquad\leq \int_0^{t_2 - t_1} \big( | f(\varphi_{0,\tau}^{ \bar{u}}(\xi'),  \bar{u}(\tau)) - f(\xi',  \bar{u}(\tau)) | + |f(\xi',  \bar{u}(\tau))| \big) \dif\tau \\
    &\qquad\leq \int_0^{t_2 - t_1} \big( L_X | \varphi_{0,\tau}^{ \bar{u}}(\xi') - \xi' | + \sup_{x,u} | f(x,u) | \big) \dif\tau \\
    &\qquad= L_X \int_0^{t_2 - t_1} \big( | \varphi_{0,\tau}^{ \bar{u}}(\xi') - \xi' | + \frac{1}{L_X} \sup_{x,u} | f(x,u) | \big) \dif\tau.
\end{align*}
Gr\"{o}nwall's inequality \citep{hirsch_1974} applied to the function
\begin{equation*}
    \mu(t) := | \varphi_{0,t}^{ \bar{u}}(\xi') - \xi' | + \frac{1}{L_X} \sup_{x,u} | f(x,u) |
\end{equation*}
yields
\begin{equation*}
    \mu(t) \leq \sup_{x,u} | f(x,u) | \frac{1}{L_X} e^{L_X t}.
\end{equation*}
Hence, we have
\begin{align*}
    |y(t_1)-y(t_2)| &= |h^\trn (\varphi^u_{0,t_1}(\xi) - \varphi^u_{0,t_2}(\xi))| \\
    &\le |h| |\varphi^u_{0,t_1}(\xi) - \varphi^u_{0,t_2}(\xi)|\\
    &\leq |h|\sup_{x,u} | f(x,u) |\frac{1}{L_X} \big(e^{L_X (t_2 - t_1)} - 1\big).
\end{align*}
Interchanging the roles of $t_1$ and $t_2$, we obtain \eqref{eq:ymod}.

For any time $t \in [0,T]$, the difference in flows due to distinct inputs $ u_1, u_2 \in C([0,T])$ is bounded by
\begin{align*}
    &\phantom{{}={}} | \varphi_{0,t}^{ u_1}(\xi) - \varphi_{0,t}^{ u_2}(\xi) | \\
    &\qquad= \Big| \int_0^t f(\varphi_{0,\tau}^{ u_1}(\xi), u_1(\tau)) - f(\varphi_{0,\tau}^{ u_2}(\xi), u_2(\tau)) \dif\tau \Big| \\
    &\qquad\leq \int_0^t \big( |f(\varphi_{0,\tau}^{ u_1}(\xi), u_1(\tau)) - f(\varphi_{0,\tau}^{ u_2}(\xi), u_1(\tau))| + |f(\varphi_{0,\tau}^{ u_2}(\xi), u_1(\tau)) - f(\varphi_{0,\tau}^{ u_2}(\xi), u_2(\tau))| \big) \dif\tau \\
    &\qquad\leq \int_0^t \big( L_X | \varphi_{0,t}^{ u_1}(\xi) - \varphi_{0,t}^{ u_2}(\xi) | + L_U | u_1(\tau) - u_2(\tau)) | \big) \dif\tau \\
    &\qquad\leq L_X \int_0^t \big( | \varphi_{0,t}^{ u_1}(\xi) - \varphi_{0,t}^{ u_2}(\xi) | + \frac{L_U}{L_X} \|  u_1 -  u_2 \|_\infty \big) \dif\tau.
\end{align*}
Gr\"{o}nwall's inequality applied to the function
\begin{equation*}
    \mu(t) := | \varphi_{0,t}^{ u_1}(\xi) - \varphi_{0,t}^{ u_2}(\xi) | + \frac{L_U}{L_X} \|  u_1 -  u_2 \|_\infty
\end{equation*}
yields
\begin{equation*}
    \mu(t) \leq \frac{L_U}{L_X} e^{L_X t} \|  u_1 -  u_2 \|_\infty,
\end{equation*}
hence we have
\begin{equation*}
    | \varphi_{0,t}^{ u_1}(\xi) - \varphi_{0,t}^{ u_2}(\xi) | \leq \frac{L_U}{L_X} \big(e^{L_X t} - 1\big) \|  u_1 -  u_2 \|_\infty.
\end{equation*}
Thus,
\begin{align*}
    \| \sG u_1 - \sG u_2 \|_\infty &= \sup_{t \in [0,T]} \left|h^\trn (\varphi_{0,t}^{ u_1}(\xi) - \varphi_{0,t}^{ u_2}(\xi))\right| \\
    &\le  \frac{|h| L_U}{L_X} \big(e^{L_X T} - 1\big) \|  u_1 -  u_2 \|_\infty,
\end{align*}
and \eqref{eq:HLip} follows.
\end{proof}
We also need the following result on the Rademacher averages of VC-subgraph classes \citep{Farrell_etal}. Recall that a class $\cG$ of measurable functions $g : \Reals^d \to \Reals$ is a \textit{VC-subgraph class} \citep[Sec.~3.6.2]{GineNickl} if the class of all sets of the form $\{ (x,r) \in \Reals^d \times \Reals : g(x) \ge r \}$ with $g \in \cG$ is a Vapnik--Chervonenkis (or VC) class, i.e., there exists a finite $D \in \Naturals$, such that, for each $m \le D$, there exist $m$ points $(x^1,r^1),\ldots,(x^m,r^m)$ that are \textit{shattered by} $\cG$, i.e.,
\begin{align*}
	\{ ({\mathbf 1}_{\{g(x^1) \ge r^1\}}, \ldots, {\mathbf 1}_{\{g(x^m) \ge r^m\}}) : g \in \cG \} = \{0,1\}^m,
\end{align*}
and no such $m$-tuple of points exists for $m > D$. This $D$ is called  the \textit{VC-subgraph dimension} (or \textit{pseudo-dimension}) of $\cG$, and is denoted by $\mathrm{vc}(\cG)$.

\begin{lemma}\label{lm:Rademacher} Let $\cG$ be VC subgraph class of real-valued measurable functions $g : \Reals^d \to [0,B]$. Let $\boldsymbol{x} = (x^1,\ldots,x^N)$ be an arbitrary $N$-tuple of points in $\Reals^d$, and define the {\em Rademacher average}
    \begin{align*}
        R_N(\cG; \boldsymbol{x}) \deq \frac{1}{N}\Exp \left[ \sup_{g \in \cG}\left| \sum^N_{i=1}\eps^i g(x^i)\right|\right],
    \end{align*}
    where $\eps^1,\ldots,\eps^N$ are i.i.d.\ random variables with $\bP[\eps^i = \pm 1] = \frac{1}{2}$. Then, for any $N \ge \mathrm{vc}(\cG)$,
    \begin{align*}
        R_N(\cG; \boldsymbol{x}) \le cB \sqrt{\frac{\mathrm{vc}(\cG) \log N}{N}}
    \end{align*}
    for some universal constant $c$.
\end{lemma}

\subsection{Proof of Theorem~\ref{thm:riskbound}}
\label{sec:proof_riskbound}

Fix any $u \in \cU$ and any $(\Sigma,\xi)$. Then 
\begin{align}\label{eq:risk_decomp}
\begin{split}
    &\phantom{{}={}} \| \sF u - \sF_{\Sigma,\xi} u \|_\infty \\
    & \qquad \le \| \sF u - (\sS^*_{k+1} \circ \sS_{k+1} \circ \sF)u \|_\infty + \| \sF_{\Sigma,\xi}u - (\sS^*_{k+1} \circ \sS_{k+1} \circ \sF_{\Sigma,\xi} \circ \sS^*_k \circ \sS_k)u \|_\infty \\
	& \qquad \qquad \fix{ + \| (\sS^*_{k+1} \circ \sS_{k+1} \circ \sF_{\Sigma,\xi} \circ \sS^*_k \circ \sS_k)u - (\sS^*_{k+1} \circ Y_{k,\Sigma,\xi} \circ \sS_k) u \|_\infty} \\
    & \qquad \qquad \fix{ + \| (\sS^*_{k+1} \circ Y_{k,\Sigma,\xi} \circ \sS_k) u - (\sS^*_{k+1} \circ \sS_{k+1} \circ \sF) u \|_\infty} \\
    & \qquad =: T_1 + T_2 + \fix{T_3 + T_4}.
\end{split}
\end{align}
Since $(\sS^*_{k+1} \circ \sS_{k+1})u = B_k u$, we can estimate $T_1$ using Lemma~\ref{lm:Bernstein} and Assumption~\ref{as:outputs}:
\begin{align}
    T_1 &\le \sup_{u \in \cU} \| \sF u - B_k(\sF u,\cdot) \|_\infty = \sup_{y \in \cY} \| y - B_k(y,\cdot) \|_\infty \le 2\omega_\cY\left(\frac{T}{\sqrt{k}}\right).\label{eq:T1}
\end{align}
For $T_2$, Lemma~\ref{lm:SS_apx} gives
\begin{align}\label{eq:T2}
    T_2 \le 2\| \sF_{\Sigma,\xi} \|_\mathrm{Lip} \omega_u\left(\frac{2T}{\sqrt{k}}\right) + 2\omega_{\sF_{\Sigma,\xi}B_{k-1}(u,\cdot)} \left(\frac{T}{\sqrt{k}}\right)
\end{align}
Now, the system \eqref{eq:RNN} has the form \eqref{eq:diffsys} with $f(x,u) = \sigma^{(n)}(Ax +b u)$ and $h = c$; thus, applying Lemma~\ref{lm:Hmap} to the i/o map $\sF_{\Sigma,\xi}$ with $\Sigma = (A,b,c)$, we get
\begin{align*}
    \| \sF_{\Sigma,\xi}\|_\mathrm{Lip} \le |c| |b|e^{\|A\|T} \qquad \text{and} \qquad
    \omega_{\sF_{\Sigma,\xi}B_{k-1}(u,\cdot)}(\delta) \le \sqrt{n}|c| e^{\|A\|T}\delta.
\end{align*}
Substituting these estimates into \eqref{eq:T2} and invoking Assumption~\ref{as:equicontinuous}, we obtain
\begin{align}\label{eq:T2_sub}
    T_2 &\le 2|c| |b| e^{\|A\|T} \omega_\cU\left(\frac{2T}{\sqrt{k}}\right) + 2|c| Te^{\|A\|T} \sqrt{\frac{n}{k}}.
\end{align}
\fix{For $T_3$, we use the fact that $(\sS^*_{k+1} \circ \sS_{k+1} \circ \sF_{\Sigma,\xi} \circ \sS^*_k \circ \sS_k)u = B_k \sF_{\Sigma,\xi} B_{k-1}u$ and $(\sS^*_{k+1} \circ Y_{k,\Sigma,\xi} \circ \sS_k)u = J^k_0 \sF_{\Sigma,\xi} B_{k-1}u$ to estimate
\begin{align}
	T_3 &= \| B_k \sF_{\Sigma,\xi} B_{k-1}u - J^k_0 \sF_{\Sigma,\xi} B_{k-1}u \|_\infty \nonumber\\
	&\le \|  B_k \sF_{\Sigma,\xi} B_{k-1}u - \sF_{\Sigma,\xi} B_{k-1}u \|_\infty + \| \sF_{\Sigma,\xi} B_{k-1}u - J^k_0 \sF_{\Sigma,\xi} B_{k-1} u \|_\infty \nonumber \\
	&\le 2\omega_{\sF_{\Sigma,\xi} B_{k-1}(u,\cdot)}\left(\frac{T}{\sqrt{k}}\right) +  \| \sF_{\Sigma,\xi} B_{k-1}u - J^k_0 \sF_{\Sigma,\xi} B_{k-1} u \|_\infty \nonumber \\
	&\le 2|c| Te^{\|A\|T} \sqrt{\frac{n}{k}} +  \| \sF_{\Sigma,\xi} B_{k-1}u - J^k_0 \sF_{\Sigma,\xi} B_{k-1} u \|_\infty.\label{eq:T3}
\end{align}}
Using \eqref{eq:T1}, \eqref{eq:T2_sub}, and \eqref{eq:T3} in \eqref{eq:risk_decomp} and taking expectation with respect to $\mu$, we obtain \eqref{eq:riskbound}.

\subsection{Proof of Theorem~\ref{thm:main}}\label{sec:proof_main}

For each pair $(\Sigma,\xi)$, define the function $g_{\Sigma,\xi} : \Reals^{k} \times \Reals^{k+1} \to \Reals_+$ according to
\begin{align*}
    g_{\Sigma,\xi}(v,z) \deq \max_{1 \le j \le k} \left| (\sS^*_{k+1} \circ Y_{k,\Sigma,\xi})v(t_j) - \sS^*_{k+1} z(t_j) \right|,
\end{align*}
where $t_j = jT/k$. Let $\bar{\mu}$ denote the joint probability law of $\sS_k u$ and $\sS_{k+1} y$ when $u \sim \mu$ and $y = \sF u$. Then $\bar{\mu}$ is a Borel probability measure on $\Reals^{k} \times \Reals^{k+1}$, and we can define the expected risk
\begin{align*}
    \bar{\cL}(\Sigma,\xi) \deq \Exp_{\bar{\mu}} [g_{\Sigma,\xi}(v,z)].
\end{align*}
Given the i/o data $(u^i,y^i) \stackrel{\mathrm{i.i.d.}}{\sim} \mu$, the points $(v^i,z^i)$ with $v^i = \sS_k(u^i)$ and $z^i = \sS_{k+1}(y^i)$ are i.i.d.\ samples from $\bar{\mu}$, and
our learning procedure selects any minimizer $(\hat{\Sigma},\hat{\xi})$ of the empirical risk
\begin{align*}
    \bar{\cL}_N(\Sigma,\xi) \deq \frac{1}{N}\sum^N_{i=1} g_{\Sigma,\xi}(v^i,z^i)
\end{align*}
among all $\Sigma = (A,b,c)$ and $\xi$ satisfying the constraint $\|A\|,|b|,|c|,|\xi| \le M$. 

We next show that, with high probability, the excess risk $\bar{\cL}(\hat{\Sigma},\hat{\xi}) - \bar{\cL}^*$ is small, where the minimum risk $\bar{\cL}^*$ is defined in \eqref{eq:minrisk}. By Lemma~\ref{lm:Hmap} applied to any i/o map $\sF_{\Sigma,\xi} \in \cF(M)$, and for any $v \in C([0,T])$,
\begin{align*}
	\| B_k \sF_{\Sigma,\xi} v \|_\infty \le \| \sF_{\Sigma,\xi}v \|_\infty \le |c|\left(|\xi| + \sqrt{n}T\right) \le M ( M + \sqrt{n}T).
\end{align*}
Moreover, for any $u \in \cU$,
\begin{align*}
    \| (\sS^*_{k+1} \circ \sS_{k+1} \circ \sF) u \|_\infty &= \| B_k(\sF u,\cdot) \|_\infty \le \|\sF u \|_\infty \le \gamma_\sF(R).
\end{align*}
Therefore, we may assume without loss of generality that $0 \le g_{\Sigma,\xi}(\cdot) \le M(M+\sqrt{n}T)+ \gamma_\sF(R) =: B$. Then the usual ERM analysis (see, e.g., Cor.~6.1 in \citet{HajekRaginsky}) guarantees that, with probability at least $1-\delta$,
\begin{align*}
    \bar{\cL}(\hat{\Sigma},\hat{\xi}) \le \bar{\cL}^* + 4 \Exp R_N(\cG) + B\sqrt{\frac{2\log(\frac{1}{\delta})}{N}},
\end{align*}
where $R_N(\cG) = R_N(\cG; ((v^1,z^1),\ldots,(v^N,z^N)))$ is the Rademacher average of the function class $\cG \deq \left\{ g_{\Sigma,\xi} : (\Sigma,\xi) \in \cF(M)\right\}$. Using Lemma~\ref{lm:Rademacher}, we then see that
\begin{align}\label{eq:ERM_bound}
    \bar{\cL}(\hat{\Sigma},\hat{\xi}) \le \bar{\cL}^* + cB\sqrt{\frac{\mathrm{vc}(\cG)\log N + \log(\frac{1}{\delta})}{N}}
\end{align}
with probability at least $1-\delta$, provided $N \ge \mathrm{vc}(\cG)$, where $c > 0$ is an absolute constant and $\mathrm{vc}(\cG)$ is the VC-subgraph dimension (or pseudo-dimension) of $\cG$.

\paragraph{Pseudo-dimension estimate.} For any $v \in \Reals^k$ and $z = (z_0,\ldots,z_k)^\trn \in \Reals^{k+1}$, we can write
\begin{align*}
    g_{\Sigma,\xi}(v,z) = \max_{j \in [k]} |h^{(j)}_{\Sigma,\xi}(v,z)|,
\end{align*}
where
\begin{align*}
    h^{(j)}_{\Sigma,\xi}(v,z) \deq \sum^{k}_{\ell = 0} \big(y^{(\ell)}_{k,\Sigma,\xi}(v) - z_\ell\big) \frac{t^\ell_j}{\ell!} ,
\end{align*}
with $y^{(\ell)}_{k,\Sigma,\xi}(v)$ denoting the $\ell$th coordinate of $Y_{k,\Sigma,\xi}(v)$. Let $\cH^{(j)}$ denote the class of all functions of the form $h^{(j)}_{\Sigma,\xi}$. We make the following two claims:
\begin{enumerate}
    \item $\cH^{(1)},\ldots,\cH^{(k)}$ are VC-subgraph classes with the same pseudo-dimension $d \le 3n^6 + 5n^3 \log_2 k$;
    \item $\cG$ is a VC-subgraph class with $\mathrm{vc}(\cG) \le 2kd$.
\end{enumerate} 
We prove the second claim first. For each $j \in [k]$, consider the class $\cC^{(j)}$ of subsets of $\Reals^{k} \times \Reals^{k+1} \times \Reals$ of the form
\begin{align}\label{eq:Hsets}
    \left\{ (v,z,r) \in \Reals^{k} \times \Reals^{k+1} \times \Reals : |h^{(j)}_{\Sigma,\xi} (v,z)| \ge r \right\}.
\end{align}
Then evidently each set $\{ (v,z,r) : g_{\Sigma,\xi}(v,z) \ge r \}$ is of the form $C^{(1)} \cup \ldots \cup C^{(k)}$ with $C^{(j)} \in \cC^{(j)}$. By the standard VC dimension estimates \citep[Prop.~3.6.7]{GineNickl}, then, $\mathrm{vc}(\cG) \le \mathrm{vc}(\cH^{(1)}) + \ldots + \mathrm{vc}(\cH^{(k)})$. On the other hand, each set in \eqref{eq:Hsets} is itself the union of the sets $\{(v,z,r) : h^{(j)}_{\Sigma,\xi}(v,z) \ge r\}$ and $\{ (v,z,r) : h^{(j)}_{\Sigma,\xi} \le - r\}$, and therefore $\mathrm{vc}(\cH^{(j)}) \le 2d$. It remains to prove the first claim.

To that end, fix some $\Sigma = (A,b,c)$ and $\xi$ and let $\theta$ be a vector of dimension $n^2 + 3n$ obtained by listing the entries of $A,b,c,\xi$ in some fixed order. Then an induction argument together with the fact that the function $\sigma(r) = \tanh r$ satisfies the identity $\sigma'(r) = 1-\sigma^2(r)$ can be used to show that each function $h^{(j)}_{\Sigma,\xi}$ can be written in the form
\begin{align*}
    P(\sigma(R_1(\theta,v,z)),\ldots,\sigma(R_n(\theta,v,z)),\theta,v,z),
\end{align*}
where $P$ is a polynomial of degree at most $3k-1$ and each $R_i$ is a polynomial of degree at most $2$ \citep{Sontag_recur}. This implies, in turn, that
\begin{align*}
    \mathrm{vc}(\cH^{(1)}) = \ldots = \mathrm{vc}(\cH^{(k)}) \le 3n^6 + 5n^3 \log_2 k
\end{align*}
\citep[Cor.~7]{Sontag_recur}. Substituting the resulting estimate of $\mathrm{vc}(\cG)$ into \eqref{eq:ERM_bound}, we see that
\begin{align}\label{eq:ERM_bound_2}
    \bar{\cL}(\hat{\Sigma},\hat{\xi}) \le \bar{\cL}^* + c(M(M+\sqrt{n}T)+ \gamma_\sF(R)) \sqrt{\frac{k(n^6 + n^3 \log_2 k)\log N + \log(\frac{1}{\delta})}{N}}
\end{align}
with probability at least $1-\delta$.

\paragraph{The final risk bound.} For any $\Sigma,\xi$ and any $u \in \cU$,
\fix{\begin{align*}
   & \| (\sS^*_{k+1} \circ Y_{k,\Sigma,\xi} \circ \sS_k) u - (\sS^*_{k+1} \circ \sS_{k+1} \circ \sF)u \|_\infty \\
	& = \| J^k_0(\sF_{\Sigma,\xi}B_{k-1}u,\cdot) - \sF_{\Sigma,\xi} B_{k-1}u,\cdot) \|_\infty \\
    &= \sup_{t \in [0,T]} | J^k_0(\sF_{\Sigma,\xi}B_{k-1}u,t) - B_k (\sF u,t) | \\
   \begin{split}
		&\le \max_{1 \le j \le k} | J^k_0(\sF_{\Sigma,\xi}B_{k-1}u,t_j) - B_k (\sF u,t_j) | \\
		&\qquad + \sup_{t \in [0,T]}\min_{1 \le j \le k} |J^k_0(\sF_{\Sigma,\xi}B_{k-1}u,t) - J^k_0(\sF_{\Sigma,\xi}B_{k-1}u,t_j)|
         + \sup_{t \in [0,T]}\min_{1 \le j \le k} |B_k(\sF u,t_j) - B_k(\sF u,t)| \\
		& \le \max_{1 \le j \le k} | J^k_0(\sF_{\Sigma,\xi}B_{k-1}u,t_j) - B_k (\sF u,t_j) | \\
		& \qquad + 2\sup_{t \in [0,T]} |J^k_0 \sF_{\Sigma,\xi}B_{k-1}u(t) - \sF_{\Sigma,\xi}B_{k-1}u(t)|  + \sup_{t \in [0,T]}\min_{1 \le j \le k} |\sF_{\Sigma,\xi}B_{k-1}u(t) - \sF_{\Sigma,\xi}B_{k-1}u(t_j)| \\
		& \qquad + \sup_{t \in [0,T]}\min_{1 \le j \le k} |B_k(\sF u,t_j) - B_k(\sF u,t)|,
   \end{split}
\end{align*}}
where the last two terms can be upper-bounded using the moduli of continuity of $B_k(\sF_{\Sigma,\xi}B_{k-1}u,\cdot)$ and $B_k(\sF u,\cdot)$, which, by Lemma~\ref{lm:Bernstein}, are upper-bounded by twice the moduli of continuity of $\sF_{\Sigma,\xi}B_{k-1}u$ and $\sF u$, respectively. Thus, using Assumption~\ref{as:outputs} and Lemma~\ref{lm:Hmap}, we get
\fix{\begin{align*}
\begin{split}
    & \| (\sS^*_{k+1} \circ Y_{k,\Sigma,\xi} \circ \sS_k) u - (\sS^*_{k+1} \circ \sS_{k+1} \circ \sF)u \|_\infty \\
    & \qquad \le \max_{1 \le j \le k} | B_k(\sF_{\Sigma,\xi}B_{k-1}u,t_j) - B_k (\sF u,t_j) |  \\
	& \qquad \qquad  +  \frac{2\sqrt{n}}{k} MT e^{MT} + 2\omega_{\cY}\left(\frac{T}{k}\right) \\
	& \qquad \qquad + 2\|J^k_0 \sF_{\Sigma,\xi} B_{k-1}u - \sF_{\Sigma,\xi}B_{k-1}u \|_\infty.
\end{split}
\end{align*}
Taking expectation of both sides with respect to $\mu$ yields, for any $(\Sigma,\xi) \in \cF(M)$,
\begin{align*}
&\Exp_\mu \left[ \| (\sS^*_{k+1} \circ Y_{k,\Sigma,\xi} \circ \sS_k) u - (\sS^*_{k+1} \circ \sS_{k+1} \circ \sF)u \|_\infty \right]  \\
& \quad \le  \bar{\cL}(\Sigma,\xi) +  \frac{2\sqrt{n}}{k} MT e^{MT} + 2\omega_{\cY}\left(\frac{T}{k}\right) + 2 \sup_{(\Sigma,\xi) \in \cF(M)} \Exp_\mu [\|J^k_0 \sF_{\Sigma,\xi}B_{k-1}u -\sF_{\Sigma,\xi}B_{k-1}u \|_\infty ].
\end{align*}
Using this and \eqref{eq:ERM_bound_2} in the bound of Theorem~\ref{thm:riskbound}, we get the statement of the theorem.}

\section*{Acknowledgments}
The work of J.~Hanson and M.~Raginsky was supported in part by the Illinois Institute for Data Science and Dynamical Systems (iDS$^2$), an NSF HDR TRIPODS institute, under award CCF-1934986 and by the Center for Advanced Electronics through Machine Learning (CAEML) NSF I/UCRC award CNS-16-24811.

\bibliography{RNN_PAC_arxiv_corr.bbl}

\end{document}